\documentclass[twoside]{article}

\usepackage[accepted]{aistats2020}
\usepackage{amsmath}
\usepackage{amssymb}
\usepackage{amsthm}
\usepackage{subcaption}
\usepackage{pgfplots}
\usepackage{thm-restate}
\usepackage[multiple]{footmisc}
\pgfplotsset{width=10cm,compat=1.9}

\pgfplotsset{
	standard/.style={
		axis x line=middle,
		axis y line=middle,
		enlarge x limits=0.15,
		enlarge y limits=0.15,
		every axis x label/.style={at={(current axis.right of origin)},anchor=north west},
		every axis y label/.style={at={(current axis.above origin)},anchor=north east},
		every axis plot post/.style={mark options={fill=white}}
}}

\newcommand{\wt}{\widetilde}

\newenvironment{proofof}[1]{\noindent{\bf Proof of #1:}}{$\qed$\par}
\theoremstyle{plain}
\newtheorem{thm}{\protect\theoremname}
\theoremstyle{plain}
\newtheorem{claim}[thm]{\protect\claimname}
\theoremstyle{plain}

\theoremstyle{plain}
\newtheorem{lem}[thm]{\protect\lemmaname}
\theoremstyle{plain}
\newtheorem{cor}[thm]{\protect\corollaryname}
\theoremstyle{definition}
\newtheorem{defn}[thm]{\protect\definitionname}
\theoremstyle{definition}

\theoremstyle{definition}

\theoremstyle{plain}

\makeatother

\providecommand{\claimname}{Claim}
\providecommand{\lemmaname}{Lemma}
\providecommand{\propositionname}{Proposition}
\providecommand{\theoremname}{Theorem}
\providecommand{\corollaryname}{Corollary}
\providecommand{\definitionname}{Definition}
\providecommand{\assumptionname}{Assumption}
\providecommand{\remarkname}{Remark}

\global\long\def\RR{\mathbb{R}}
\global\long\def\CC{\mathbb{C}}
\global\long\def\ZZ{\mathbb{Z}}

\global\long\def\R{{\cal R}}

\global\long\def\x{{\mathbf{x}}}
\global\long\def\w{{\mathbf{w}}}

\global\long\def\a{{\mathbf{a}}}

\global\long\def\t{{\mathbf{t}}}
\global\long\def\y{{\mathbf{y}}}
\global\long\def\EE{{\mathbb{E}}}
\global\long\def\w{{\mathbf{w}}}

\global\long\def\z{{\mathbf{z}}}

\global\long\def\sinc#1{\mathrm{sinc}\left(#1\right)}

\newcommand*{\rect}{\mathrm{rect}}

\newcommand{\wh}{\widehat}

\newcommand{\Fc}{\mathcal{F}}

\newcommand{\bs}[1]{\boldsymbol{#1}}
\newcommand{\bv}[1]{\mathbf{#1}}

\global\long\def\veta{{\bs{\eta}}}
\global\long\def\vxi{{\bs{\xi}}}
\global\long\def\vgamma{{\bs{\gamma}}}

\usepackage[round]{natbib}

\begin{document}

%

%

\runningauthor{Michael Kapralov, Navid Nouri, Ilya Razenshteyn, Ameya Velingker, Amir Zandieh}
\twocolumn[
\aistatstitle{Scaling up Kernel Ridge Regression via Locality Sensitive Hashing}
\aistatsauthor{Michael Kapralov \And Navid Nouri \And Ilya Razenshteyn}
\aistatsaddress{EPFL \And EPFL \And Microsoft Research}
\aistatsauthor{Ameya Velingker \And Amir Zandieh}
\aistatsaddress{Google Research \And EPFL}
]

\begin{abstract}
  Random binning features, introduced in the seminal paper of Rahimi and Recht (2007), are an efficient method for approximating a kernel matrix using locality sensitive hashing. Random binning features provide a very simple and efficient way of approximating the Laplace kernel but unfortunately do not apply to many important classes of kernels, notably ones that generate {\em smooth} Gaussian processes, such as the Gaussian kernel and Mat\'{e}rn kernel. In this paper we introduce a simple {\em weighted} version of random binning features, and show that the corresponding kernel function generates Gaussian processes of any desired smoothness. We show that our weighted random binning features provide a spectral approximation to the corresponding kernel matrix, leading to efficient algorithms for kernel ridge regression. Experiments on large scale regression datasets show that our method outperforms the accuracy of random Fourier features method.
\end{abstract}

\section{Introduction} \label{sec:intro}

Kernel methods are a powerful framework for applying non-parametric modeling techniques to a number of problems in statistics and machine learning, such as ridge regression, SVM, PCA, etc.  While kernel methods have been well studied and are capable of achieving excellent empirical results, they often pose scalability challenges as they operate on the \emph{kernel matrix} (Gram matrix) $K$ of the data, whose size scales up quadratically in the number of training instances. Thus, much work has focused on scaling up kernel methods by producing suitable approximations to the kernel or its underlying kernel matrix.

One such approach for scaling up kernel methods was shown by \citet{DBLP:conf/nips/RahimiR07}, who showed how to approximate positive definite shift-invariant kernels using \emph{random binning features}. The idea is to partition an input space into randomly shifted grids and map input points into bins such that the probability that two input points $\x$ and $\y$ are mapped to the same bin is proportional to $k(\x, \y)$. This enables one to get an estimator for $k(\x, \y)$ by counting the number of times $\x$ and $\y$ are binned together.

The above approach can also be viewed in the context of \emph{locality sensitive hashing (LSH)}~\citep{indyk-motwani98, Har-PeledIM12}, an algorithmic technique that hashes elements of an input space into ``buckets'' such that similar input items are hashed into the same buckets with high probability. More specifically, the hash collision probability between two items is desired to be proportional to the similarity index of the items, i.e., collisions should be more likely for more similar items. LSH has found practical uses for a number of problems such as nearest neighbor search, clustering, etc. The random binning features of \citet{DBLP:conf/nips/RahimiR07} can be viewed as an LSH scheme in which the similarity measure is the kernel.

\citet{DBLP:conf/nips/RahimiR07} show that random binning features yield an \emph{unbiased} estimator $\widetilde{k}(\x,\y)$ for $k(\x,\y)$, provided that $k$ satisfies certain conditions. They also establish \emph{point-wise} concentration of $\widetilde{k}$ to $k$, but in many numerical linear algebra applications, point-wise concentration is insufficient. On the other hand, \emph{spectral guarantees} for the kernel matrix $K$, whose $(\x,\y)$-entry is given by $k(\x,\y)$, are a popular sufficient condition that guarantees various statistical and algorithmic implications. One such guarantee is captured by the (regularized) \emph{oblivious subspace embedding (OSE)} property, as stated below.

\begin{restatable}[Oblivious subspace embedding (OSE)]{defn}{osedef}
	\label{OSE-definition}
	Given $\epsilon, \delta, \lambda > 0$, and the positive semi-definite matrix $K \in \RR^{n \times n}$, an \emph{$(\epsilon, \delta, \lambda)$-oblivious subspace embedding (OSE)} for this kernel matrix is a distribution $\mathcal{D}$ over $n \times n$ matrices $\widetilde{K}$ such that with probability at least $1 - \delta$,
	\begin{equation}\label{eq:aspectral}
	(1-\epsilon)(K + \lambda I) \preceq \widetilde{ K} + \lambda I \preceq (1+\epsilon) (K + \lambda I).
	\end{equation}
\end{restatable}

\paragraph{Kernel Ridge Regression (KRR).}
One popular kernel method for which OSE has algorithmic implications is the problem of \emph{kernel ridge regression} (KRR), which we focus on in this work. In KRR, one is given labeled training data $(\x_1, y_1), (\x_2, y_2), \dots, (\x_n, y_n) \in \RR^d \times \RR$ and a \emph{regularization parameter} $\lambda > 0$, and the response of an input vector $\x$ is estimated as follows:
\[
 \overline{\eta} (\x) = \sum_{j=1}^n k(\x_j, \x) \alpha_j,
\]
where ${\boldsymbol\alpha} = (\alpha_1 \cdots \alpha_n)^T$ is the solution of the equation $(K + \lambda I_n){\boldsymbol\alpha} = \mathbf{y}$, where $\mathbf{y} = (y_1 \cdots y_n)^T$ and $I_n$ is the $n\times n$ identity matrix. Solving this matrix equation generally requires $\Theta(n^3)$ time and $\Theta(n^2)$ memory, which is impractical for large datasets. Thus, the design of scalable methods for KRR and other kernel methods has been the focus of much recent research~\citep{bach2013sharp, CdV07, alaoui2015fast, ZhangDW15, musco2017recursive, AvronCW17, avron2017random}.

The OSE property for $\widetilde{K}$ is useful because it allows $\widetilde{K} + \lambda I_n$ to be used as an effective preconditioner for the solution of the aforementioned matrix equation, while enabling one to bound the excess risk~\citep{avron2017random}. Thus, the approach we take is to find a new class of estimators that satisfies the OSE property while enabling fast matrix-vector computation.

\paragraph{WLSH estimators.}
Our main contribution is to formulate a new class of estimators, which we term \emph{Weighted LSH (WLSH) estimators}, that generalize the random binning features of \citet{DBLP:conf/nips/RahimiR07} and applies to a wider range of kernels. More specifically, given a probability density function $p(\cdot)$ with non-negative support over $\RR$ and a \emph{bucket-shaping function} $f(\cdot)$ (see discussion below), we can define a kernel with kernel matrix $K \in \RR^{n\times n}$ as well as a corresponding WLSH estimator.

Our first main theorem shows that appropriately many independent instances $\widetilde{K}^1, \widetilde{K}^2, \dots, \widetilde{K}^m$ of the WLSH estimator yield an OSE $\widetilde{K}$ for $K$:
\begin{equation}
\widetilde{K} = \frac{1}{m} \sum_{s=1}^m \widetilde{K}^s \label{eq:lsh-ose}
\end{equation}

\begin{thm}[Main Theorem, \emph{informal version of Theorem~\ref{lem:OSE}}]\label{thrm2-informal}
 Let $\x^1, \x^2, \dots, \x^n \in \RR^d$ be a collection of points and $\epsilon,\lambda > 0$. For any $p(\cdot)$ and any $f$ supported on $[-1/2,1/2]$ with $\|f\|_2 = 1$, the distribution $\widetilde{K}$ given by \eqref{eq:lsh-ose} is an $(\epsilon, 1/\mathrm{poly}(n), \lambda)$-OSE for $K$, provided that the number of independent instances of the WLSH estimator is $m = \Omega\left( \frac{\|f^{\otimes d}\|_\infty^2}{\epsilon^2} \cdot \frac{n}{\lambda} \cdot \log n\right)$.
\end{thm}
Our WLSH estimator reduces to standard random binning features when the \emph{bucket-shaping function} $f$ is chosen to be a \emph{rectangle function} $\rect$ supported on $[-1/2, 1/2]$. However, the generalization of allowing different bucket-shaping functions $f$ enables the estimator to be applicable to a wider range of kernels, which we discuss below.

Standard random binning features work only for certain classes of shift-invariant kernels $k(\cdot)$ that satisfy a convex decomposition property~\citep{DBLP:conf/nips/RahimiR07}. The Laplace kernel, given by $k(\x,\y) = k(\x-\y) = \exp(-|\x-\y|)$, is an important example of such a kernel. However, note that the Laplace kernel does not satisfy \emph{smoothness}, which is often a desired property. Indeed, the limitation of random binning features to non-smooth kernels is inherent, as any suitable shift-invariant kernel $k(\cdot)$ must have the property that $1-k(\cdot)$ satisfies the triangle inequality~\citep{charikar2002similarity}. This precludes the possibility of using random binning features to approximate any monotonically decreasing smooth kernel that is twice differentiable.

The non-smoothness limitation arises from the fact that the bins in random binning are discontinuous at the edges, as the shape of the corresponding bins is a rectangle. Our approach circumvents this limitation by generalizing random binning features to an estimator that allows ``soft'' buckets with smoother edges (specified by the bucket-shaping function $f$ in Theorem~\ref{lem:OSE}). This allows us to construct new families of smooth kernels that can be estimated using our WLSH estimators but are not amenable to standard random binning features.

We complement Theorem~\ref{thrm2-informal} with a lower bound showing that the number of instances of the WLSH Estimator in Theorem~\ref{thrm2-informal} is essentially tight:
\begin{thm}[Main Theorem, \emph{informal version of Theorem~\ref{thm:lowerbound}}]
  Let $p(w) = we^{-w}$ be the PDF for the Gamma distribution, and let $f(\cdot) = \rect(\cdot)$ be the bucket-shaping function. For any $\lambda > 0$, $d\geq 1$, and $n\geq 8\lambda$, there exists a dataset $\x^1, \x^2, \dots, \x^n \in \RR^d$ such that in order for $\widetilde{K}$ given by \eqref{eq:lsh-ose} to be an $(\epsilon, 1/n, \lambda)$-OSE for $K$ with $\epsilon \leq 1/6$, one requires $m = \Omega\left(\frac{1}{\epsilon^2}\cdot \frac{n}{\lambda}\cdot \log n\right)$ independent instances of the WLSH estimator.
\end{thm}

Furthermore, our WLSH estimator allows $\widetilde{K}$ to be stored with little memory while supporting fast matrix-vector multiplication, which allows it to be suitable for KRR. In this direction, we conduct a number of experiments on various large-scale regression datasets that show the accuracy and speed of approximate KRR using our WLSH kernels and estimator compared to exact KRR and other popular approximation methods. The results show that our WLSH-based method produces better accuracy than the popular method of random Fourier features on large datasets while still offering favorable running times. We additionally present experiments showing the performance of our WLSH-based kernel family for learning Gaussian processes through KRR.

\subsection{Related work}
Another line of work for producing low-rank approximations to kernel matrices is the Nystr\"{o}m method. A number of works have sought to improve the method using leverage score sampling, risk inflation bounds, etc.~\citep{bach2013sharp, alaoui2015fast, rudi2015less, musco2017recursive}.
Although there has been much work on kernel approximation sketches that achieve the optimal target dimension, e.g., Nystr\"{o}m sampling~\citep{musco2017recursive}, all such methods that are known are \emph{data-dependent}, barring any strong assumptions on the kernel matrix. \emph{Data-oblivious} approaches, on the other hand, have the advantage of being implementable in distributed settings. WLSH estimators (and random binning features), being OSEs, fall into this paradigm.

There are a number of works on devising OSEs. Most of these are related to the technique of \emph{Random Fourier features}, which was also introduced by \citet{DBLP:conf/nips/RahimiR07} and provides a popular data-oblivious approach for kernel approximation. \citet{avron2017random} showed that a modification of Random Fourier features yields provably better target dimension. \citet{AhleKKPVWZ20} improved upon this result and were able to embed the Gaussian kernel in Euclidean space with a target dimension that is not exponential in the dimension of the dataset. However, some Gaussian processes that arise in practice are less smooth than those arising from Gaussian kernels, and the result of \citet{AhleKKPVWZ20} does not extend to the Laplace kernel~\footnote{One can trivially use the result of \citet{AhleKKPVWZ20} for Laplace kernels by using a trivial embedding of $\ell_1$ norms into $\ell_2$, but this results in a blowup in dimension that is impractical} or Mat\'{e}rn kernels.

\section{Preliminaries} \label{sec:prelim}
In this section we introduce notations and present basic definitions and claims.

	The \emph{Fourier transform} of a continuous function $g : \RR^d \to \CC$ in
	$L_1(\RR^n)$ is defined to be the function $\mathcal{F}g: \RR^d \to \CC$ given
	by
	$(\mathcal{F}g)(\bs{\xi}) =  \int_{\RR^d} g(\bv{t}) e^{-2\pi i \bv{t}^\top \bs{\xi}}\,d\bv{t}$.
	We also sometimes use the notation $\wh{g}$ for the Fourier transform of $g$. We often informally refer to $g$ as representing the function in \emph{time domain} and $\hat{g}$ as representing the function in \emph{frequency domain}.
	The original function $g$ can also be obtained from $\hat{g}$ by the
	\emph{inverse Fourier transform}: $g(\bv{t}) = \int_{\RR^d} \wh{g}(\bs{\xi}) e^{2\pi i \bs{\xi}^\top \bv{t}}\,d\bs{\xi}$. The \emph{convolution} of two functions $g:\RR^d\to\CC$ and $h:\RR^d\to\CC$ is
	defined to be the function $(h*g):\RR^d\to\CC$ given  by $(h*g)(\veta) = \int_{\RR^d} h(\bv{t})g(\veta-\bv{t})\,d\bv{t}$ for $\veta\in\RR^d$. We use $\delta_d$ to denote the d-dimensional \emph{Dirac delta function}.

We now define the \emph{rectangle function} (boxcar).
\begin{defn}[Rectangle Function] \label{def:rect}
	For any $a > 0$ we define the 1-dimensional \emph{rectangle function} $\rect_a:\RR\to\CC$ as
	\[
	\rect_a(x) = \begin{cases} 0\qquad &\text{if $|x| > a/2$}\\  1\qquad &\text{if $|x| \le a/2$} \end{cases}.
	\]
	If $a=1$, we omit the subscript and just write $\rect$.
\end{defn}

	For any vector $\w=(w_1,w_2,\dots, w_d)^\top$ we use the notation $[0,\w]$ to denote the set $[0,w_1] \times [0,w_2] \times \dots\times [0,w_d]$. Moreover, if ${\bf j} = (j_1,j_2, \dots, j_d)^\top$, then we use the notation ${\bf j} \w = (j_1w_1, j_2w_2, \dots, j_dw_d)^\top$ and ${\bf j} /\w = (j_1/w_1, j_2/w_2, \dots, j_d/w_d)^\top$.
	Also, for any function $f : \RR \rightarrow \RR$ the notation $f^{\otimes d}$ denotes the function $f^{\otimes d} : \RR^d \rightarrow \RR$, defined as $f^{\otimes d}(\x) = \prod_{l=1}^d f(x_l)$ for every $\x \in \RR^d$.

\section{Weighted Locality Sensitive Hashing (WLSH) estimator} \label{sec:lsh-estimator}
In this section we first provide background on random binning features and Locality Sensitive Hashing and then define our WLSH estimator in Section \ref{sec:lsh-exp} and prove its smoothness properties in Section \ref{sec:smoothness}.
	\emph{Random binning features} were introduced by \citet{DBLP:conf/nips/RahimiR07} as an estimator for a certain class of kernel functions such as the Laplace kernel. The main building block of this estimator is a \emph{Locality Sensitive Hashing (LSH)} family, defined as follows: 
\begin{defn}[Locality Sensitive Hash Family]\label{lsh-hashfunction-def}
	For any positive integer $d$, we define the Locality Sensitive Hash (LSH) family $\mathcal{H}$ as the collection of hash functions,
	$\mathcal{H}:= \left\{h_{\w,\z}(\cdot) : \w \in \RR_{+}^d, \z \in [0,\w] \right\}$, where the LSH function $h_{\w,{\z}} : \RR^d \rightarrow \ZZ^d$ is given by,
	\begin{equation}\label{eq:lsh-hash}
	[h_{\w,{\z}}(\x)]_l = \text{round}\left(\frac{x_l - {z}_l}{w_l} \right),
	\end{equation}
	for every $l \in [d]$ and $\x \in \RR^d$.
	The parameters of the LSH functions in this family are distributed as follows: $\w=(w_1,w_2, \dots,w_d)^\top$ is a random vector with iid entries $w_1,w_2,\dots, w_d \sim p(w)$ for some probability distribution $p(\cdot)$ with non-negative support and $\z$ is a uniform random vector in $[0,\w]$.
\end{defn}

Random binning features are given by the following estimator:
\begin{equation}\label{rbf-estimator}
\widetilde{k}(\x,\y) = \begin{cases}
1 &\text{ if } h_{\w,{\z}}(\x) = h_{\w,{\z}}(\y)\\
0 &\text{ otherwise}
\end{cases},
\end{equation}
where $h_{\w,{\z}}(\x) \sim \mathcal{H}$ is an LSH function.
Note that the expectation of this estimator is equal to the collision probability of the LSH function $h_{\w,\z}$, i.e., $\EE\left[\widetilde{k}(\x,\y)\right]=\Pr_{h_{\w,{\z} }\sim \mathcal{H}}[h_{\w,{\z}}(\x)=h_{\w,{\z}}(\y)]$.
It is shown in \citet{DBLP:conf/nips/RahimiR07} that if $\mathcal{H}$ is the LSH family given in Definition \ref{lsh-hashfunction-def} with $p(w) = w e^{-w}$ (Gamma distribution), then the collision probability of two points $\x,\y$ is $\EE_{h_{\w,{\z} }\sim \mathcal{H}}\left[\widetilde{k}(\x,\y)\right]=e^{-\|\x-\y\|_1}$, which is the Laplace kernel.
The Laplace kernel is non-smooth due to the discontinuity of its derivative at the origin. There is a great deal of interest in using smooth kernels in many machine learning applications \citep{srinivas2009gaussian}.
By changing the distribution over the LSH family $\mathcal{H}$ via varying the PDF $p(w)$, one can obtain the random binning feature estimator for some class of kernels. One might hope to find a distribution over $\mathcal{H}$ such that $\EE_{h_{\w,{\z}}(\cdot) \sim \mathcal{H}}[\widetilde{k}(\x,\y)]$ gives a smooth kernel such as the Squared exponential kernel or Mat\'{e}rn kernel.
But it follows from \citet{charikar2002similarity} that the random binning feature is only able to approximate kernel functions $k(\cdot)$ such that $1-k(\x-\y)$ satisfies the triangle inequality. This requirement is very restrictive and leaves the random binning features inapplicable to the most popular classes of smooth kernels including the Squared exponential kernel and Mat\'{e}rn family. In fact, any smooth kernel which is monotonically decreasing and is at least twice differentiable cannot be approximated using random binning features.

The random binning features estimator is an estimator whose output is either zero or one. We generalize this in Section \ref{sec:lsh-exp} by allowing the estimator to assume a range of values and show that this estimator, unlike the random binning features estimator, is able to approximate a rich family of smooth kernels.
\subsection{WLSH kernel family} \label{sec:lsh-exp}
We now define the \emph{Weighted LSH (WLSH) Estimator}.
\begin{defn}[WLSH Estimator]
	\label{def:-lsh-estimator} 
	Let $f: \RR \rightarrow \RR$ be some even function with support $[-1/2, 1/2]$ and $\|f\|_2=1$ and let $p(\cdot)$ be some PDF with non-negative support. Also let $\mathcal{H}$ be the LSH family as in Definition~\ref{lsh-hashfunction-def}. For any $\x,\y \in \RR^d$, the \emph{Weighted LSH (WLSH) estimator} $\widetilde{k}_{f,p}$ is defined as:
	\begin{equation}\label{def-lshestimator-eq}
	\widetilde{k}_{f,p}(\x,\y) = \begin{cases}
	A &\text{ if } h_{\w,{\z}}(\x) = h_{\w,{\z}}(\y)\\
	0 &\text{ otherwise}
	\end{cases},
	\end{equation}
	where $A = f^{\otimes d} ( h_{\w,{\z}}(\x)  + \frac{\z - \x}{\w} ) \cdot f^{\otimes d} ( h_{\w,{\z}}(\y)  + \frac{\z - \y}{\w} )$, and $h_{\w,\z} \sim \mathcal{H}$. 
\end{defn}

For ease of notation, we often drop the subscripts and just write $\widetilde{k}(\cdot)$ to denote the WLSH.
We show that the expectation of the WLSH estimator is a valid shift-invariant kernel.
The expectation of the estimator is given by the following claim,

\begin{claim} \label{claim:expectation-estimator-fourier} For any PDF $p(\cdot)$ with non-negative support, any even function $f: \RR \rightarrow \RR$ with support $[-1/2, 1/2]$ and $\|f\|_2=1$, and any $\x,\y \in \RR^d$, the expectation of the WLSH kernel $\widetilde{k}(\x,\y)$ over the random choice of LSH function $h_{\w,\z} \sim \mathcal{H}$ is given by
\begin{align*}
&\EE_{h_{\w,\z} \sim \mathcal{H} } \left[\widetilde{k}(\x,\y)\right] \\
&=  \int_{\RR^d} e^{2\pi i (\x - \y)^\top \vxi} \prod_{l=1}^d  \EE_{w_l \sim p(w)}\left[w_l \cdot \left|\wh{f}({w_l \xi_l})\right|^2\right] d\vxi.
\end{align*}	
Equivalently, it can be expressed as
\begin{multline*}
\EE_{h_{\w,\z} \sim \mathcal{H} } \left[\tilde{k}(\x,\y)\right]\\ =  \prod_{l=1}^d  \EE_{w_l \sim p(w)}\left[ (f * f) \left(\frac{x_l-y_l}{w_l}\right) \right].
\end{multline*}
\end{claim}

By Claim~\ref{claim:expectation-estimator-fourier}, $\EE \left[\tilde{k}(\x,\y)\right]$ is clearly shift-invariant. Moreover, by the convolution theorem (see Claim~\ref{claim:convthm}), the Fourier transform of the expectation is
\begin{multline*}
\Fc\left[ \EE \left[\tilde{k}(\cdot+\y,\y)\right] \right](\vxi)\\ = \prod_{l=1}^d  \EE_{w_l \sim p(w)}\left[w_l \cdot \left|\wh{f}({w_l \xi_l})\right|^2\right],
\end{multline*}
which is a positive function for every $\vxi$. Hence, the expectation of the WLSH kernel is a valid kernel.
We now formally define \emph{WLSH kernels families}.

\begin{defn}[WLSH Kernel Family]\label{def:gen-kernel}
	Let $p(\cdot)$ be some probability density function with support $\RR_{+}$ and let $f: \RR \rightarrow \RR$ be some even function with support $[-1/2, 1/2]$ and $\|f\|_2=1$. The WLSH kernel function $k_{f,p}:\RR^d \rightarrow \RR$ is defined as
	$$k_{f,p}(\x) = \prod_{l=1}^{d} \left(\int_{0}^\infty {p(w_l)}\cdot (f * f)\left( \frac{x_l}{w_l} \right) \, dw_l\right),$$
	for any $\x\in \RR^d$. We often drop the subscripts $f,p$ and just write $k(\cdot)$ to denote the WLSH kernel.
\end{defn}

It follows from Claim \ref{claim:expectation-estimator-fourier} that for any WLSH kernel $k(\cdot)$, there exists an unbiased WLSH estimator
\[\EE_{h_{\w,\z} \sim \mathcal{H}} \left[\widetilde{k}(\x,\y)\right] = k(\x-\y).\]

\subsection{Smoothness of WLSH Gaussian process}\label{sec:smoothness}

In the context of Bayesian estimation, some regularity assumptions are often made about the function being learned. Smoothness is the most common assumption. 
Suppose that $\eta : \RR^d \rightarrow \RR$ is a sample path from a \emph{Gaussian process} GP$(0, k(\x-\y))$, i.e., its mean is $\mathbb{E}[\eta(\x)] = 0$ for every $\x \in \RR^d$ and its covariance is given by the kernel function $\mathbb{E}[ \eta(\x) \eta(\y) ] = k(\x-\y)$ for every $\x,\y \in \RR^d$, where $k(\cdot)$ is a shift-invariant positive definite kernel. The Bayesian estimation algorithms commonly assume that the sample paths of the GP, $\eta(\x)$ satisfy certain smoothness properties with high probability. For instance, in the context of Gaussian process optimization in bandit setting, to get a provable guarantee, the known algorithms require the derivatives of the GP's sample path, $\frac{\partial \eta(\mathbf{x})}{\partial \mathbf{x}}$, to be bounded everywhere with sub-Gaussian tail probability \cite{srinivas2009gaussian}. We prove that our WLSH construction (Definition \ref{def:gen-kernel}) provides a class of smooth kernels.

In the following lemma we prove that the sample paths of GP$(0,k_{f,p}(\x-\y))$ when the covariance $k_{f,p}(\cdot)$ is WLSH kernel (Definition \ref{def:gen-kernel}) inherit their smoothness from the bucket-shaping function $f$. The lemma shows that our construction of \emph{WLSH} family of kernels is able to generate a GP such that the partial derivatives of a sample path from this GP is bounded everywhere with a sub-Gaussian distribution as long as the function $f(\cdot)$ is smooth. As shown in Figure \ref{LSH-quadatic-form}, we use a bucket shape $f(\cdot)$ which has a smooth transition around the edges as opposed to random binning features whose bucket shape is $\rect(\cdot)$ with a discontinuity at the edges.
Here we denote the partial derivative with respect to $j^\text{th}$ coordinate $\partial/\partial_j$ by $D_j$. The partial derivative of the GP with respect to the $j^\text{th}$ coordinate is denoted by $D_j\eta(\x)$. The sample paths of this process are $D_j\eta(\x)$, where $\eta(\x)$ is a sample path from the original GP.

\begin{lem}\label{lem:smooth-GP-lsh-kernel}
For any positive integer $q$, any integers $q_1, q_2, \cdots q_d \ge 0$ such that $\sum_j q_j = q$ let the derivative operator ${\bf D}$ be defined as ${\bf D} = D_1^{q_1} D_2^{q_2} \cdots D_d^{q_d}$. For any even function $f$ with support $[-1/2, 1/2]$ which has bounded derivatives of up to $q+1$ order and any PDF $p(\cdot)$ with non-negative support, if $\eta:[0,1]^d \rightarrow \RR$ is a sample path from GP$(0,k(x-y))$, where $k(\cdot)$ is the WLSH kernel (Definition \ref{def:gen-kernel}), then the mixed partial derivative of the sample path, ${\bf D} \eta(x)$, satisfies the following high probability bound:
$$\Pr\left[ \sup_{\x \in [0,1]^d} \left| {\bf D} \eta(\x)\right| > M \right] \le \left(\frac{L M}{\sigma^2}\right)^d e^{-\frac{M^2}{\sigma^2}},$$
where $\sigma^2 = \prod_{l=1}^d \left\|f^{(q_l)}\right\|_2^2 \int_{\RR_{+}} {\frac{p(w_l)}{w_l^{2q_l}}}  \, dw_l$ and $L = O\left( \sup_{{ j \in [d]}} \left| \prod_{\substack{ l \in [d] }} \left\|f^{(q_l+\delta_{l,j})}\right\|_2^2 \int_{\RR_{+}} {\frac{p(w_l)}{w_l^{2(q_l+\delta_{l,j})}}} \, dw_l \right| \right)$ where $\delta_{l,j} = 0$ for every $l \neq j$ and $\delta_{j,j}=1$.
\end{lem}

\section{Spectral approximation and Kernel Ridge Regression (KRR)} 
In this section we prove our main results which show that our weighted LSH estimator provides an OSE for kernel matrices.
Suppose that you are given a collection of points in the $d$ dimensional Euclidean space $\x^1, \x^2, \ldots \x^n \in \RR^d$ together with (noisy) measurements of some unknown function $\eta^*: \RR^d \rightarrow \RR$,
$$\gamma_i = \eta^*(\x^i) + \epsilon_i,$$
where the $\epsilon_i$ are iid Gaussians with variance $\sigma_\epsilon^2$ and the aim is to estimate the underlying function $\eta^*(\x)$ from the data. One simple yet powerful method for solving this problem is the \emph{Kernel Ridge Regression} (KRR).
To find the KRR estimator, one needs to solve the least squares problem $\min_{\beta} \left\| K \beta - \vgamma \right\|_2^2 + \lambda \beta^\top K \beta$, where $K \in \RR^{n \times n}$ is the kernel matrix defined as $K_{ij} = k(\x^i,\x^j)$ and $\vgamma=(\gamma_1, \cdots \gamma_n)^\top$. The least squares solution is $\beta^* =  \left( K + \lambda I \right)^{-1} \vgamma$. If the function $\eta^*$ is a sample path from a GP$(0,k(\x,\y))$ then the KRR estimator (i.e., $\eta(\cdot) = \sum_{i \in [n]}\beta_i^* k(\cdot,\x^i)$) is optimal in the Bayesian sense.

In order to accelerate the computational complexity KRR, 
we approximate the kernel function $k(\cdot)$ using the WLSH estimator (Definition \ref{def:-lsh-estimator}). For any $\x^1,\x^2, \dots, \x^n \in \RR^d$, the approximated kernel matrix $\widetilde{K} \in \RR^{n \times n}$ is defined as,
$[\widetilde{K}]_{ij} = \widetilde{ k}_{f,p}(\x^i,\x^j)$, where $\widetilde{ k}_{f,p}(\cdot)$ is the WLSH estimator as in Definition~\ref{def:-lsh-estimator}.
One can see that the matrix $\widetilde{K}_{f,p}$ is very structured and typically sparse (it's $ij^\text{th}$ entry is nonzero only if $\x^{i}$ and $\x^{j}$ get hashed into the same bucket, i.e., $h_{\w,\z}(\x^i) = h_{\w,\z}(\x^j)$). Hence, $\widetilde{K}_{f,p}$ supports fast matrix vector multiplication and can be stored in small memory.

\paragraph{Approximate kernel matrix $\widetilde{K}$ can be stored in small memory and supports fast matrix vector multiplication:} 
Suppose that we want to build a data structure which can be stored in space $O(n)$ such that using this data structure we can compute the product $\widetilde{K} \beta$ for arbitrary vectors $\beta \in \RR^n$ in linear time $O(n)$. It follows from Definition~\ref{def:-lsh-estimator} that for any $s \in [n]$,
$$(\widetilde{K} \beta)_s =  B_{h_{\w,{\z}}(\x^s)}(\beta)  \cdot f^{\otimes d} \left( h_{\w,{\z}}(\x^s) + \frac{{\z} - \x^s}{\w} \right),$$
where $B_{\bf j}(\beta) = \sum_{i: h_{\w,\z}(\x^i) = {\bf j}} \beta_i \cdot f^{\otimes d} \left( {\bf j} + \frac{{\z} - \x^i}{\w} \right)$ for every bucket ${\bf j}$ and we call it the \emph{load} of bucket ${\bf j}$. 
This is illustrated in Figure \ref{LSH-quadatic-form} for the one dimensional case. In dimension one, to compute the load of $j^\text{th}$ bucket, we first shift the function $f$ to $z + jw$ and then for every $x^i$ which is hashed into $j^\text{th}$ bucket, we scale $\beta_i$ by the function value at point $x^i$, $f(\frac{x^i - jw - z}{w})$, and sum them all up. 

Therefore we construct the data structure as follows: We first hash all the data points $\x^i$ using the LSH function $h_{\w,{\z}}(\cdot)$ and keep the lists $L_{\bf j_1},L_{\bf j_1}, \dots$, where each list corresponds to one of the non-empty buckets of this hashing. Each list $L_{\bf j_r}$ contains the points $\x^i$ which are hashed to bucket ${\bf j_r}$, i.e., $L_{\bf j_r} = \{ i : h_{\w,{\z}}(\x^i) = {\bf j_r} \}$ for every $r$. All the lists can be formed in time $O(dn)$ which is the time to hash all data points. And the total size of all lists is the number of data points $n$, because each data point gets hashed into exactly one bucket, hence the data structure can be stored using $O(n)$ memory words. Then to compute the product $\widetilde{K} \beta$ first we compute the bucket load $B_{\bf j_r}(\beta)$ for every non-empty bucket~${\bf j_r}$, 
$$B_{\bf j_r}(\beta) = \sum_{i \in L_{\bf j_r}} \beta_i \cdot f_d \left( {\bf j_r} + \frac{{\z} - \x^i}{\w} \right).$$
We can do this for all buckets using time $O(n)$. Then every coordinate $s$ of the product $(\widetilde{K} \beta)_s$ is computed as follows:
$$(\widetilde{K} \beta)_s = B_{h_{\w,{\z}}(\x^s)}(\beta) \cdot f^{\otimes d} \left( h_{\w,{\z}}(\x^s) + \frac{{\z} - \x^s}{\w} \right),$$
where $B_{h_{\w,\z}(\x^s)}(\beta)$ denotes the load of the bucket $\x^s$ is hashed into.
Hence, the product can be computed in total time $O(n)$.

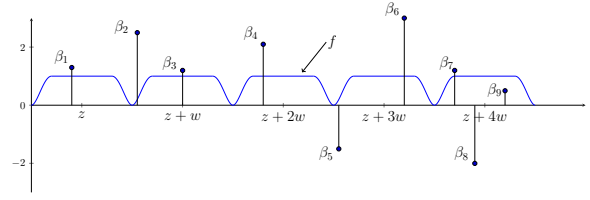
\begin{figure}[h]
	
	\scalebox{0.40}{
		\begin{tikzpicture}
		\begin{axis}[
		restrict y to domain=-10:10,
		samples=1000,
		width=20cm, height=210pt,
		ymin=-3 ,ymax=3,
		xmin=-2, xmax=3.5,
		xtick={-1.5,-0.5,...,2.5},
		xticklabels={\Large{${z}$},\Large{${z}+w$},\Large{${z}+2w$},\Large{${z}+3w$},\Large{${z}+4w$}},
		axis x line=center,
		axis y line=left,
		]
		\addplot [
		domain=0:0.2, 
		samples=100, 
		color=blue, thick
		]
		{((5*x)^2)*((2-(5*x))^2)};
		\addplot [
		domain=0.2:0.8, 
		samples=100, 
		color=blue, thick
		]
		{ 1};
		\addplot [
		domain=0.8:1, 
		samples=100, 
		color=blue, thick
		]
		{((5*(1-x))^2)*((2-(5*(1-x)))^2)};
		
		\addplot [
		domain=1:1.2, 
		samples=100, 
		color=blue, thick
		]
		{((5*(x-1))^2)*((2-(5*(x-1)))^2)};
		\addplot [
		domain=1.2:1.8, 
		samples=100, 
		color=blue, thick
		]
		{ 1};
		\addplot [
		domain=1.8:2, 
		samples=100, 
		color=blue, thick
		]
		{((5*(1-(x-1)))^2)*((2-(5*(1-(x-1))))^2)};

		\addplot [
		domain=2:2.2, 
		samples=100, 
		color=blue, thick
		]
		{((5*(x-2))^2)*((2-(5*(x-2)))^2)};
		\addplot [
		domain=2.2:2.8, 
		samples=100, 
		color=blue, thick
		]
		{ 1};
		\addplot [
		domain=2.8:3, 
		samples=100, 
		color=blue, thick
		]
		{((5*(1-(x-2)))^2)*((2-(5*(1-(x-2))))^2)};

		\addplot [
		domain=-1:-0.8, 
		samples=100, 
		color=blue, thick
		]
		{((5*(x+1))^2)*((2-(5*(x+1)))^2)};
		\addplot [
		domain=-0.8:-0.2, 
		samples=100, 
		color=blue, thick
		]
		{ 1};
		\addplot [
		domain=-0.2:0, 
		samples=100, 
		color=blue, thick
		]
		{((5*(1-(x+1)))^2)*((2-(5*(1-(x+1))))^2)};

		\addplot [
		domain=-2:-1.8, 
		samples=100, 
		color=blue, thick
		]
		{((5*(x+2))^2)*((2-(5*(x+2)))^2)};
		\addplot [
		domain=-1.8:-1.2, 
		samples=100, 
		color=blue, thick
		]
		{ 1};
		\addplot [
		domain=-1.2:-1, 
		samples=100, 
		color=blue, thick
		]
		{((5*(1-(x+2)))^2)*((2-(5*(1-(x+2))))^2)};

		\end{axis}

		\begin{axis}
		[
		samples=2,
		domain = -2:-1,
		width=20cm, height=210pt,
		ymin=-3 ,ymax=3,
		xmin=-2, xmax=3.5,
		xtick={-1.5,-0.5,...,2.5},
		xticklabels={${z}$,${z}+w$,${z}+2w$,${z}+3w$,${z}+4w$},
		hide axis
		]
		\addplot+[ycomb,black, thick] table [x={n}, y={xn}] {data.dat};
		\end{axis}
		\draw(1, +4.5) node {\Large $\beta_1$};
		\draw(3, +5.5) node {\Large $\beta_2$};
		\draw(4.6, +4.3) node {\Large $\beta_3$};
		\draw(7.3, +5.3) node {\Large $\beta_4$};
		\draw(9.8, 1.3) node {\Large $\beta_5$};
		\draw(12, +6.1) node {\Large $\beta_6$};
		\draw(13.8, +4.3) node {\Large $\beta_7$};
		\draw(14.3, 1.3) node {\Large $\beta_8$};
		\draw(15.4, 3.4) node {\Large $\beta_9$};
		\draw(10, 5) node {\Large $f$};
		\draw [->,thick] (9.8,5) -- (9,4);

		\end{tikzpicture}
	}
	
	\caption{The \emph{load} of ${\bf j}^\text{th}$ bucket corresponds to shifting the bucket-shaping function $f^{\otimes d}$ to ${\bf j}\w+{\z}$ and then integrating it against $\alpha(\x) = \sum_{j=1}^n \beta_j \delta(\x - \x^j)$.} \label{LSH-quadatic-form}
\end{figure}

\subsection{Oblivious subspace embedding via WLSH estimator}\label{sec:ose}
Recall that our aim is to solve the least squares problem $\min_{\beta} \left\| K \beta - \vgamma \right\|_2^2 + \lambda \beta^\top K \beta$ quickly by using an approximate kernel matrix $\widetilde{ K}$. In order to get a provable $(1\pm\epsilon)$-approximate solution to the least squares problem, $\widetilde{ K}$ must be spectrally close to original $K$ in some way. In this paper we focus on \emph{oblivious subspace embeddings} (see Definition \ref{OSE-definition}) and show that this property is enough to get a provably good approximation to the least squares problem.
We need the following claim before proving the main result,

\begin{claim}\label{claim:spectral-nomr-estimator}
	For any dataset $\x^1,\x^2,\dots,\x^n \in \RR^d$, if $k(\cdot)$ is the WLSH estimator as in Definition \ref{def:-lsh-estimator} then its corresponding kernel matrix $\widetilde{K} \in \RR^{n \times n}$, is symmetric and satisfies,
	$0 \preceq \widetilde{K} \preceq n \|f^{\otimes d}\|_\infty^2 \cdot  I$.
\end{claim}

Now we are ready to prove the main theorem and show that WLSH estimator provides an oblivious subspace embedding for  WLSH kernel matrix $K$.

\begin{thm}\label{lem:OSE}

	For any positive integers $d,n$, any collection of points $\x^1,\x^2,\dots, \x^n \in \RR^d$, any PDF $p(\cdot)$ with non-negative support, any even function $f(\cdot)$ with support $[-1/2,1/2]$ and $\|f\|_2=1$, let $k(\cdot)$ be the WLSH kernel as in Definition \ref{def:gen-kernel} and let $K \in \RR^{n \times n}$ be its kernel matrix.
	If $\widetilde{ k}^{1}(\cdot), \widetilde{k}^{2}, \dots, \widetilde{ k}^{m}(\cdot)$ are independent instances of WLSH estimator as per Definition \ref{def:-lsh-estimator} and $\widetilde{K}^1,\widetilde{K}^2,\dots, \widetilde{K}^m$ are their kernel matrices, then for any $\lambda, \epsilon > 0$, the matrix $\widetilde{ K} := \frac{1}{m}\sum_{s=1}^{m} \widetilde{K}^s$ is an $\left( \epsilon, \frac{1}{\mathrm{poly}(n)}, \lambda \right)$-oblivious subspace embedding (see Definition \ref{OSE-definition}) for the kernel matrix $K$ as long as $m = \Omega\left( \frac{\|f^{\otimes d}\|_\infty^2}{\epsilon^2} \cdot ({n}/{\lambda})\cdot \log n\right)$.
\end{thm}
\begin{proof}
	Let $U \in \RR^{n \times n}$ be the unitary matrix of eigenvectors of $K$, i.e., $i^\text{th}$ column of matrix $U$ corresponds to $i^\text{th}$ eigenvector of matrix $K$ (The eigenvalues are ordered in the decreasing order $\lambda_1 \ge \lambda_2 \ge \dots \ge \lambda_n \ge 0$). Since $U$ is unitary ($U^\top U=I_{n\times n}$), it is enough to prove that with probability $1 - \frac{1}{\mathrm{poly}(n)}$, 
	$(1-\epsilon) U^\top(K + \lambda I)U \preceq U^{\top}(\wt{K} + \lambda I)U \preceq (1+\epsilon) U^\top(K+\lambda I)U$.
	Let $Z= (U^\top (K+\lambda I) U )^{-1/2}$. Since $Z$ is a diagonal matrix with entries $Z_{i,i}=\frac{1}{\sqrt{\lambda_i+\lambda}}$ and is, therefore, positive definite, we can multiply the above identity from left and right by $Z$ and equivalently prove that,
	$(1-\epsilon) I \preceq Z^\top U^{\top}(\wt{K} + \lambda I)U Z \preceq (1+\epsilon) I$.
	In order to satisfy the above it is sufficient to have $||Z^\top U^{\top}(\wt{K} + \lambda I)U Z - I||_{op} \le \epsilon$ where $\|\cdot\|_{op}$ denotes the operator norm of matrices. Therefore, it suffices to prove
	$\Pr\left[ \left\|Z^\top U^{\top}(\wt{K} + \lambda I)U Z - I \right\|_{op} \le \epsilon \right] \ge 1 - \frac{1}{\mathrm{poly}(n)}$,
	which follows from the matrix Chernoff bound of Lemma~\ref{lem-matrix-chenoff} (see Appendix~\ref{sec:appD}).

	By Claim \ref{claim:spectral-nomr-estimator} the estimators $\widetilde{K}^s$ are PSD, therefore, $0 \preceq Z^\top U^\top \left(\widetilde{K}^s + \lambda I \right) U Z$, 
	for every $s \in [m]$. Also because of the unbiasedness of estimators, $\mathbb{E} \left[ Z^\top U^\top \left(\widetilde{K}^s + \lambda I\right) U Z \right] = I$.
	Therefore we can invoke Lemma~\ref{lem-matrix-chenoff}. In order to do so, we need to upper bound the operator norm of $Z^\top U^\top \left(\widetilde{K}_{f,p}^s + \lambda I \right) U Z$. By Claim \ref{claim:spectral-nomr-estimator}, we have $\left\|\widetilde{K}^s + \lambda I \right\|_{op} = \left\| \widetilde{K}^s \right\|_{op} + \lambda \le n \cdot \|f^{\otimes d}\|_\infty^2 + \lambda$; thus,
	\begin{align*}
	&\left\|Z^\top U^{\top} \left( \widetilde{K}^s + \lambda I \right)U Z \right\|_{op}\\
	&\le \left\|\widetilde{K}^s + \lambda I \right\|_{op} \cdot \|Z^\top U^{\top}U Z \|_{op}\\
	&\le \left( n \|f^{\otimes d}\|_\infty^2 + \lambda \right) \cdot ||Z^\top Z||_{op} \le \frac{n}{\lambda} \cdot \|f^{\otimes d}\|_\infty^2 +1.
	\end{align*}
	The result now follows by Lemma~\ref{lem-matrix-chenoff} (see Appendix~\ref{sec:appD}).	
\end{proof}

Now we show that our analysis in Theorem~\ref{lem:OSE} is not loose and in order to get an OSE for worst case datasets, one needs $m = \Omega \left( \frac{1}{\epsilon^2} (n/\lambda) \log n \right)$.

\begin{thm}[Lower Bound in order to achieve OSE] \label{thm:lowerbound} 
	Let $f(\cdot) = \rect (\cdot)$ and $p(w) = w e^{-w}$ (Gamma distribution) and let $k(\cdot)$ be the WLSH kernel as in Definition \ref{def:gen-kernel}. For any integer $d\ge 1$ any $\lambda > 0 $ and any integer $n \ge 8 \lambda$, there exists a dataset $\x^1, \cdots \x^n \in \RR^d$ such that if $K \in \RR^{n \times n}$ is the kernel matrix defined as $K_{ij} = k_{f,p}(\x^i - \x^j)$ and $\widetilde{ k}^{1}(\cdot), \widetilde{ k}^{2}, \cdots \widetilde{ k}^{m}(\cdot)$ are independent instances of WLSH estimator as per Definition \ref{def:-lsh-estimator} and $\widetilde{K}^1,\widetilde{K}^2,\cdots \widetilde{K}^m$ are their kernel matrices then for any $0 < \epsilon \le 1/6$ in order for $\widetilde{ K} := \frac{1}{m}\sum_{s=1}^{m} \widetilde{K}^s$ to be an $(\epsilon, \frac{1}{n}, \lambda)$-oblivious subspace embedding for $K$ one needs to have $m =  \Omega \left( \frac{1}{\epsilon^2} \cdot \frac{n}{\lambda} \cdot \log n \right)$.
\end{thm}
{\bf Proof sketch:} let the points $\{\x^i\}_{i=1}^n \subseteq \RR^d$ be positioned as $\x^1 =\dots = \x^{n/2} = (-\lambda/n, 0,0, \dots 0)^\top$ and $\x^{n/2+1} =\cdots = \x^{n} = (\lambda/n, 0, 0, \cdots 0)^\top$.
Let the vector $\beta \in \CC^n$ be defined as,
$\beta_1 =\beta_2=\cdots = \beta_{n/2} = -1$
and 
$\beta_{n/2+1} =\dots = \beta_{n} = 1$. 
The proof proceeds by showing that in order to preserve the quadratic form corresponding to this $\beta$, one needs to set $m =  \Omega \left( \frac{1}{\epsilon^2} \cdot \frac{n}{\lambda} \cdot \log n \right)$.
By some calculations, we see that $\beta^\top \widetilde{ K}^s \beta$ has the following distribution:
	$$\beta^\top \widetilde{ K}^s \beta = \begin{cases}
	\frac{n^2}{2} & \text{with probability } p \le  \frac{2 \lambda}{n}\\
	0 & \text{with probability } 1-p
	\end{cases}.$$
	Thus, to obtain a non-zero estimator with constant probability, one needs $m = \Omega(\frac{n}{\lambda})$. In order to obtain the $(1\pm\epsilon)$-approximation guarantee with high probability, the bound improves by a factor of $\frac{1}{\epsilon^2} \log n$ (see Appendix~\ref{sec:appD}).

\subsection{Approximate KRR via WLSH}\label{sec:approx-KRR-lsh}
In this section we give the algorithm for approximate KRR problem using the WLSH estimator. Let $\widetilde{ k}_{f,p}^s(\cdot)$ be independent instances of the WLSH estimator for all $s \in [m]$. We define the approximate kernel function $\widetilde{ k}(\cdot) := \frac{1}{m} \sum_{s=1}^m \widetilde{ k}_{f,p}^s(\cdot)$ and let $\widetilde{ K}$ be the corresponding kernel matrix. Suppose $\eta^*: \RR \rightarrow \RR$ is the underlying function to be learned via KRR and the measurements are $\gamma_i = \eta^*(\x^i) + \epsilon_i$,
where $\epsilon_i$'s are iid normal noise with variance $\sigma_\epsilon^2$.
We solve the approximate regressor by solving the linear system,
$(\widetilde{K} + \lambda I) \beta = { \vgamma}$,
where $ { \vgamma} = (\gamma_1, \dots, \gamma_n)^\top$. Then the approximate regressor estimates the function values at a point $\x$ as follows:
\begin{align*}
\widetilde{ \eta}(\x) &= \sum_{i \in [n]} \beta_i\widetilde{ k}(\x,\x^i)\\&= \frac{1}{m} \sum_{s=1}^m B_{h_{\w,{\z}}^s(\x)}(\beta) \cdot f^{\otimes d} \left( h_{\w,{\z}}^s(\x) + \frac{{\z} - \x}{\w} \right) 
\end{align*}
where $B_{h_{\w,{\z}}^s(\x)}(\beta) = \sum_{i: h_{\w,{\z}}^s(x^i) = h_{\w,{\z}}^s(\x)} \beta_i \cdot f^{\otimes d} \left( h_{\w,{\z}}^s(\x^i) + \frac{{\z} - \x^i}{\w} \right)$ is the load of the bucket that $\x$ gets hashed into via $s^\text{th}$ LSH function, $h_{\w,{\z}}^s$.

We give the \emph{empirical risk bound} for the WLSH estimator in Appendix~\ref{sec:appF}.

\section{Experiments} \label{sec:experiments}

\paragraph{Estimating a GP using the WLSH kernel:}

In the first set of experiments we show that our WLSH kernel family from Section \ref{sec:lsh-estimator} performs as accurately as the most popular kernel functions for learning Gaussian processes through KRR. Specifically, we generate a random function $\eta:[0,1]^d \rightarrow \RR$ which is a sample path from a Gaussian process with zero mean whose covariance $\sigma(x,y)=\EE[\eta(x)\eta(y)]$ is one of {\bf (1)} Laplace $ e^{-\|x-y\|_1}$ or {\bf (2)} Squared Exponential $ e^{-\|x-y\|_2^2}$ or {\bf (3)} Mat\'{e}rn with $\nu=5/2$: $C_{5/2}(x-y) = \left(1+\|x-y\|_2+\|x-y\|_2^2/3\right)e^{-\|x-y\|_2}$.

\begin{table}[h]
	\centering
	\caption{Test set RMSE for estimating GPs.} \label{fig-synthetic}
	\begin{center}
		\scalebox{0.65}{
			\begin{tabular}{ p{2cm}|p{1.0cm}||p{1.4cm}|p{2.3cm}|p{1.4cm}|p{1.2cm}  }
				
				{\bf Covariance of GP $\sigma(\cdot)$} & {\bf Dim.} &{\bf Laplace} &{\bf Squared exponential} & {\bf Mat\'{e}rn $\nu=5/2$}&{\bf WLSH $k_{f,p}(\cdot)$}\\
				\hline
				& $30$    &0.128&   0.086&   0.093&   0.088\\
				$e^{-{\|\cdot\|_2^2}}$&   $5$  & 0.043   &0.031&   0.032&   0.029\\
				\hline
				&$30$ & 0.385&  0.479&   0.481&   0.438\\
				$e^{-{\|\cdot\|_1}}$  &   $5$  & 0.103&0.230&   0.226&   0.166\\
				\hline
				& $30$  & 0.335   &0.291&   0.299&   0.294\\
				$C_{5/2}(\cdot)$& $5$ & 0.013&0.016&   0.013&   0.012\\
				
			\end{tabular}
		}
	\end{center}
\end{table}

We run this experiment for two settings: Low-dimensional data ($d=5$) and high-dimensional data ($d=30$). In each case, we sample $\eta(\x)$ uniformly over $[0,1]^d$ at $4000$ points. We use $3000$ samples for training the estimator and $1000$ samples for testing. Then we estimate the function value on test data using KRR on the training data. We run KRR with various kernel function choices and show that our WLSH kernel (Definition \ref{def:gen-kernel}) performs as well as the most popular kernel functions such as Mat\'{e}rn $\nu=5/2$, Squared Exponential, and Laplace. The WLSH kernel we used for this experiment has the bucket-shaping function $f(x) = \left(\rect * \rect_{1/4} * \rect_{1/4} \right)(2x)$. This function has a continuous derivative and a bounded second derivative. Moreover, we chose the PDF to be $p(w) = \frac{w^6}{5!}e^{-w}$. Thus, the resulting kernel has bounded mixed partial derivatives of up to the fourth order. This is the same type of smoothness as the Mat\'{e}rn kernel with $\nu=5/2$, but in our experiments (see Table \ref{fig-synthetic}), we outperform Mat\'{e}rn kernel on all datasets. Moreover, in the low-dimensional setting $d=5$, we outperform the Squared Exponential kernel.

\paragraph{Large scale KRR on real data:}
Our second set of experiments shows that the WLSH estimator speeds up KRR on standard real data sets by orders of magnitude compared to exact KRR and has better accuracy than the popular Random Fourier Features (RFF) \cite{DBLP:conf/nips/RahimiR07}. We evaluate the following methods:

{\bf Exact KRR} using exact kernel computation for various shift-invariant kernel functions.

{\bf Random Fourier Features (RFF)} for approximating the squared exponential kernel. The kernel value is approximated by $ \widetilde{k}(x^i,x^j) = \phi(x^i)^\top \phi(x^j)$, where $\phi:\RR^d \rightarrow \RR^D$ is a random mapping and $D$ denotes the number of random features.

{\bf WLSH} using the procedure explained in Section~\ref{sec:approx-KRR-lsh} with bucket-shaping function $f(\cdot) = \rect (\cdot)$ and PDF $p(w) = w e^{-w}$. 

\paragraph{Results:} The Root Mean Square Error (RMSE) of different methods on the test data set as well as the time to train the regressors are presented in Table \ref{table:realdata}.\footnote{All methods require solving a linear system which we do using the Conjugate Gradient method. The most expensive computation in each iteration is multiplying a vector by the (approximate) kernel matrix. This takes time $\approx n^2$ for exact methods and time $\approx n D$ for RFF, where $D$ is the number features, and time $\approx n m$ for WLSH method, where $m$ is the number of LSH functions.}\footnote{Since RFF and LSH method are randomized, we ran the experiments with 5 different random seeds and reported the avg. RMSE and running time in Table \ref{table:realdata}.}
One can see the LSH method is as accurate as the exact KRR on the first two datasets while its running time is at least 3x faster. On the last two datasets, the exact method did not converge to a solution within $12$ hours but the approximate methods could run pretty fast. The LSH method outperforms the accuracy of RFF on the large scale datasets. RFF requires a large number of features $D$ in order to be accurate which leads to a huge memory usage therefore on the large scale datasets where we have a memory constraint and cannot use large $D$, RFF's performance deteriorate. The running time of RFF is better than LSH method because its implementation can be optimized but when data is large and there is a memory constraint, RFF performs worse than LSH.

\begin{table}[h]
	\centering
	\caption{Test set RMSE of different regression methods together with the running times.} \label{table:realdata}
	\begin{center}
		\scalebox{0.63}{
			\begin{tabular}{ p{3.1cm}||p{1.4cm}|p{1.5cm}|p{1.4cm}|p{1.5cm}|p{1.3cm}  }
				
				{\bf Dataset} & {\bf Exact Laplace} &{\bf Exact Squared Exp.} &{\bf Exact Mat\'{e}rn $\nu=\frac{5}{2}$} & {\bf Random Fourier Features} &{\bf WLSH}\\
				\hline
				Wine Quality& 0.684 &0.728 &  0.709 & 0.737   & 0.701   \\
				$d=11$& 28 sec & 30 sec & 1 min & 2 sec  & 5 sec  \\
				size: $6497$&  &  &  & D=7000 & m=450  \\
				\hline
				Insurance Company & 0.231 & 0.231 & 0.231 & 0.231  & 0.232  \\
				$d=85$ & 3 min & 3 min & 5.5 min & 3 sec  & 2 sec  \\
				size: $9822$ &  &  &  & D=5000  & m=250  \\
				\hline
				CT Slices Location & N/A & N/A & N/A & 4.10  & 3.45  \\
				$d=384$ & \textgreater 12 hrs & \textgreater  12 hrs & \textgreater 12 hrs & 0.5 min  & 1 min  \\
				size: $53500$ & &  & & D=3500  & m=50  \\
				\hline
				Forest Cover& N/A & N/A & N/A & 0.968  & 0.720  \\
				$d=54$ & \textgreater 12 hrs & \textgreater 12 hrs & \textgreater 12 hrs & 6 min  & 7.5 min\\
				size: $581012$ &  &  &  & D=1500  & m=50\\
				
			\end{tabular}
		}
	\end{center}
\end{table}
We use the following standard large-scale regression datasets for Gaussian process regression:
The first dataset we used for regression is the {\bf Wine Quality} dataset. The dimensionality of this dataset is $d=11$. We used $4000$ samples for training the regressors and $2497$ samples for testing the accuracy.
The second dataset is {\bf Insurance Company} dataset. The dimensionality of this dataset is $d=85$. We used $5822$ samples for training the regressors and $4000$ samples for testing the performance of estimators.
The third dataset is the {\bf Location of CT Slices}. The dimensionality of this dataset is rather high $d=384$. We used $35000$ samples for training the regressors and $18500$ samples for testing their performance. 
The last dataset is the {\bf Forest Cover} dataset. The dimensionality of this dataset is $d=54$. We used $500000$ samples for training the regressors and $81012$ samples for testing the regressors.

\clearpage
\appendix
\section{Basic lemmas and claims}

The convolution theorem shows that the Fourier transform of the convolution of
two functions is simply the product of the individual Fourier transforms:
\begin{claim}[Convolution Theorem] \label{claim:convthm}
	Given functions $f:\RR^d\to\CC$ and $g:\RR^d\to\CC$ whose convolution is $h =
	f*g$, we have
	\[
	\wh{h}(\bs{\xi}) = \wh{f}(\bs{\xi})\cdot\wh{g}(\bs{\xi})
	\]
	for all $\bs{\xi}\in\RR^d$.
\end{claim}

It is not hard to see that the Fourier transform of a $\delta_d$ is the constant function which is $1$ everywhere:
\[
(\Fc\delta_d)(\vxi) = \int_{\RR^d} e^{-2\pi i \t^\top \vxi}\cdot \delta_d(\t)\,d\t
= e^{-2\pi i \cdot 0^\top \cdot \vxi} = 1
\]
for all $\vxi$. Similarly, the Fourier transform of a shifted delta function is as follows:
\begin{align*}
	(\Fc\delta(\cdot\, -\, \a))(\vxi) &= \int_{\RR^d} e^{-2\pi i \t^\top \vxi}\cdot \delta_d(\t-\a)\,d\t\\
 &= e^{-2\pi i \a^\top\vxi}.
\end{align*}
Thus, by the convolution theorem, we obtain the following identity:
\begin{claim}\label{cl:shift}
	Given a function $f:\RR^d\to\CC$, we have
	\begin{align*}
		(\Fc f(\cdot\, - \, \a))(\vxi) &= (\Fc (f * \delta_d(\cdot - \a)))(\vxi)\\
		&= \hat{f}(\vxi)
	\cdot e^{-2\pi i \a^\top \vxi}.
	\end{align*}
\end{claim}
Similarly,
\begin{claim}
	Given a function $f:\RR^d\to\CC$, we have
	\[
	(\Fc (f(\x)\cdot e^{2\pi i \a^\top \x}))(\vxi) = \hat{f} (\vxi - \a).
	\]
\end{claim}

\begin{claim}\label{claim:stretch-variable}
	For any function $g: \RR^d \rightarrow \RR$, and any $\w \in \RR_{+}^d$ the following holds,
	$$\Fc \left[ g(\frac{\cdot}{\w}) \right](\vxi) =\left( \prod_{l=1}^d w_l\right) \wh{g}(\w \vxi).$$
\end{claim} 

Finally, we introduce a useful function known as the \emph{Dirac comb function}:
\begin{defn}
	For any ${\bf T} \in \RR_{+}^d$ the d-dimensional \emph{Dirac comb function} with period ${\bf T}$ is defined as
	$f$ satisfying
	\[
	f(\x) =  \sum_{\mathbf{j} \in \ZZ^d} \delta(\x - \mathbf{j}  {\bf T}),
	\]
	where $\mathbf{j}  {\bf T} = (j_1T_1, j_2T_2, \cdots j_dT_d)^\top$.
\end{defn}
We use the Dirac comb function in our lower bound constructions.
It is a standard fact that the Fourier transform of a Dirac comb function is
another Dirac comb function which is scaled and has the inverse period:
\begin{claim}\label{diracF}
	Let
	\[
	f(\x) = \sum_{\mathbf{j} \in \ZZ^d} \delta(\x - \mathbf{j}  {\bf T})
	\]
	be the d-dimensional Dirac comb function with period ${\bf T}$. Then,
	\[
	(\Fc f)(\vxi) = \prod_{l=1}^{d}\frac{1}{T_l} \cdot \sum_{\mathbf{j} \in \ZZ^d} \delta\left(\vxi - \frac{\mathbf{j}}{{\bf T}}\right),
	\]
	where $\frac{\mathbf{j}}{{\bf T}} = (j_1/T_1, j_2/T_2, \dots, j_d/T_d)^\top$.
\end{claim}

\begin{claim}[Nyquist-Shannon]\label{claim:nyquist}
	Given a function $f:\RR^d\to\CC$, we have:
	\begin{align*}
	&\Fc \left( f(\cdot)  \sum_{\mathbf{j}\in\ZZ^d}
	\delta_d(\cdot-{\bf w} \cdot \mathbf{j}) \right)(\bs{\xi})\\ 
	&\qquad= \left(\prod_{i=1}^{d}w_i^{-1}\right) \sum_{\mathbf{j}\in\ZZ^d}  \Fc(f)(\bs{\xi}-\mathbf{j}/{\bf w}).
	\end{align*}
\end{claim}

\section{Omitted claims and proofs from Section~\ref{sec:lsh-exp}}

We use the following basic claim about Fourier transform of Nyquist-Shannon sampling of functions.
\begin{claim}\label{cl:basic}
	For any ${\bf w} \in \RR_{+}^d$, every sequences $\{c_{\bf j}\}_{{\bf j} \in \ZZ^d}$ and $\{b_{\bf j}\}_{{\bf j} \in \ZZ^d}$ such that $\sum_{{\bf j} \in \ZZ^d} |c_{\bf j}|^2<\infty$ and $\sum_{{\bf j} \in \ZZ^d} |b_{\bf j}|^2<\infty$, if $g({\cdot})=\sum_{{\bf j} \in \ZZ^d} c_{\bf j}\delta(\cdot - {\bf j}  {\w}) $ and $h({\cdot})=\sum_{{\bf j} \in \ZZ^d} b_{\bf j}\delta(\cdot - {\bf j}  {\w}) $, then the following conditions hold.
	\begin{description}
		\item[(1)] $\wh{g}({\vxi})=\sum_{{\bf j} \in \ZZ^d} c_{\bf j} \exp(-2\pi i \vxi^\top ({\bf j}{\w} ))$ for every ${\vxi}\in \RR^d$.
		\item[(2)] $c_{\bf j}=\left(\prod_{l=1}^{d} w_l\right) \int_{[0,1/{\bf w}]} \wh{g}(\vxi) \exp(2\pi i ({\bf j}{\w})^\top \vxi )\, d\vxi$ for every ${\bf j} \in \ZZ^d$. 
		\item[(3)] $\sum_{{\bf j} \in \ZZ^d} c_{\bf j}^* b_{\bf j} = \left(\prod_{l=1}^{d} w_l\right) \int_{[0,1/{\bf w}]} \wh{g}(\vxi)^* \wh{h}(\vxi) \, d\vxi$.
	\end{description}
\end{claim}  

\begin{proof}
	
	We have 
	\begin{equation*}
	\begin{split}
	\wh{g}(\vxi)&=\int_{\RR^d} \left(\sum_{{\bf j} \in \ZZ^d} c_{\bf j} \delta(\x-{\bf j}{\w})\right) \exp(-2\pi i \vxi^\top \x)\, d\x\\
	&=\sum_{{\bf j} \in \ZZ^d} c_{\bf j} \int_{\RR^d}  \delta(\x-{\bf j}{\w}) \exp(-2\pi i \vxi^\top \x) \, d\x\\
	&=\sum_{{\bf j} \in \ZZ^d} c_{\bf j}  \exp(-2\pi i  \vxi^\top( {\w}{\bf j})),\\
	\end{split}
	\end{equation*}
	which gives the first claim. The second claim can be verified directly:
	\begin{equation*}
	\begin{split}
	&\int_{[0,1/{\bf w}]} \wh{g}(\vxi) \exp(2\pi i ({\w}{\bf j})^\top \vxi) \,  d\xi\\
		&=\int_{[0,1/{\bf w}]} \sum_{{\bf l} \in \ZZ^d} c_{\bf l} e^{2\pi i ({\w}{\bf j})^\top \vxi - 2\pi i \vxi^\top ({\w}{\bf l})} \, d\vxi\\
		&=\sum_{{\bf l} \in \ZZ^d}  c_{\bf l} \int_{[0,1/w]^d} e^{2\pi i ({\w}{\bf j})^\top \vxi - 2\pi i \vxi^\top ({\w}{\bf l})}  \, d\vxi \\
	&=\left(\prod_{i=1}^{d}\frac1{w_i}\right) c_{\bf j},
	\end{split}
	\end{equation*}
	proving the second claim.
	
	For the third claim we have
	\begin{equation*}
	\begin{split}
	&\int_{[0,1/{\bf w}]} \wh{g}(\vxi)^* \wh{h}(\vxi) \, d\xi\\
	&=\int_{[0,1/{\bf w}]} \sum_{{\bf j} \in \ZZ^d}\sum_{{\bf l} \in \ZZ^d} c_{\bf j}^* b_{\bf l}  \exp(-2\pi i  \vxi^\top ({\w}({\bf l-j}))) \, d\vxi\\
	&=\sum_{{\bf j} \in \ZZ^d}\sum_{{\bf l} \in \ZZ^d} c_{\bf j}^* b_{\bf l} \int_{[0,1/{\w}]}  \exp(-2\pi i  \vxi^\top ({\w}({\bf l-j}))) \, d\vxi\\
	&=\left(\prod_{i=1}^{d}\frac1{w_i}\right) \sum_{{\bf j} \in \ZZ^d} c_{\bf j}^* b_{\bf j},
	\end{split}
	\end{equation*}
	as required.
\end{proof}

The properties of the WLSH estimator are best understood using the means of Fourier transform. Therefore, we express the WLSH estimator in the Fourier domain. The following lemma expresses the WLSH estimator in the spectral domain.

\begin{lem}[Spectral Representation of WLSH Estimator] \label{lem:spectral-quadratic-form}
	For any $\w \in \RR_{+}^d$, any ${\z}\in [0, {\w}]$, any $\x,\y \in \RR^d$, if the WLSH estimator $\widetilde{k}_{f,p}(\x,\y)$ is defined as in \eqref{def-lshestimator-eq}, then the following holds,
	\begin{equation}\label{estimator-fourier}
	\tilde{k}_{f,p}(\x,\y) = \left(\prod_{l=1}^{d} w_l\right) \int_{[0,1/{\bf w}]} \wh{g}(\vxi)^* \wh{h}(\vxi) \, d\vxi,
	\end{equation}
	where $\wh{g}(\vxi) = \sum_{{\bf j}\in\ZZ^d} e^{-2\pi i (\x - {\z})^\top (\vxi-\frac{\bf j}{\w})}\cdot \wh{f}^{\otimes d}({\w \vxi-{\bf j}})$ and $\wh{h}(\vxi) = \sum_{{\bf j}\in\ZZ^d} e^{-2\pi i (\y - {\z})^\top (\vxi-\frac{\bf j}{\w})}\cdot \wh{f}^{\otimes d}({\w \vxi-{\bf j}})$.
\end{lem}

\begin{proof}
	By \eqref{def-lshestimator-eq} and part {\bf (3)} of Claim~\ref{cl:basic} we have,
	
	\begin{align}
	\widetilde{k}_{f,p}(\x,\y) &= \sum_{{\bf j} \in \ZZ^d} f^{\otimes d} ( {\bf j}  + \frac{{\z} - \x}{\w} ) \cdot f^{\otimes d} ( {\bf j}  + \frac{{\z} - \y}{\w} )\nonumber\\
	&= \left(\prod_{l=1}^{d} w_l\right) \int_{[0,1/{\bf w}]} \wh{g}(\vxi)^* \wh{h}(\vxi) \, d\vxi,\label{eq:q3rfdef}
	\end{align}
	where $g(\t) = (\delta_d(\cdot - \x + {\z})*f^{\otimes d}(\frac{\cdot}{\w}) ) \cdot \sum_{{\bf j} \in \ZZ^d} \delta(\t - {\bf j}{\w})$ and $h(\t) = (\delta_d(\cdot - \y + {\z})*f^{\otimes d}(\frac{\cdot}{\w}) ) \cdot \sum_{{\bf j} \in \ZZ^d} \delta(\t - {\bf j}{\w})$.
	We now apply Claim~\ref{claim:nyquist} to $g(\cdot)$ and $h(\cdot)$, obtaining the following for every $\vxi \in \RR^d$,
	\begin{align}	
	\wh{g}(\vxi) &= \prod_{l=1}^{d} \frac{1}{w_l}\sum_{{\bf j}\in\ZZ^d} \Fc \left[ f^{\otimes d} \left(\frac{\cdot - \x + {\z}}{\w} \right) \right](\vxi-{\bf j}/{\w}).
	\end{align}
	Similarly,
	\[
		\wh{h}(\vxi) = \prod_{l=1}^{d} \frac{1}{w_l}\sum_{{\bf j}\in\ZZ^d} \Fc \left[ f^{\otimes d}\left( \frac{\cdot - \y + {\z}}{\w} \right) \right](\vxi-{\bf j}/{\w}).
	\]
	Now by Claim \ref{claim:convthm} and Claim~\ref{claim:stretch-variable},
	\begin{multline*}
		\Fc \left[ f^{\otimes d}( \frac{\cdot - \x + {\z}}{\w} ) \right](\vxi)\\ = \prod_{l=1}^{d} w_l \cdot e^{-2\pi i (\x - {\z})^\top \vxi} \wh{f}^{\otimes d}(\w \vxi).
	\end{multline*}
	And similarly $\Fc[f^{\otimes d}\left(\frac{\cdot - \y + {\z}}{\w} \right)](\vxi) = \prod_{l=1}^{d} w_l \cdot e^{-2\pi i (\y - {\z})^\top \vxi} \wh{f}^{\otimes d}({\w \vxi})$.
	Substituting these into~\eqref{eq:q3rfdef}, we get
	\begin{equation*}
	\widetilde{k}_{\w,{\z}}(\x,\y) = \left(\prod_{l=1}^{d} w_l\right) \int_{[0,1/{\bf w}]} \wh{g}(\vxi)^* \wh{h}(\vxi) \, d\vxi,
	\end{equation*}
	where
	\[
		\wh{g}(\vxi) = \sum_{{\bf j}\in\ZZ^d} e^{-2\pi i (\x - {\z})^\top (\vxi-\frac{\bf j}{\w})}\cdot \wh{f}^{\otimes d}({\w \vxi-{\bf j}})
	\]
	and
	\[
		\wh{h}(\vxi) = \sum_{{\bf j}\in\ZZ^d} e^{-2\pi i (\y - {\z})^\top (\vxi-\frac{\bf j}{\w})}\cdot \wh{f}^{\otimes d}({\w \vxi-{\bf j}}).
	\]
\end{proof}

\begin{proofof}{Claim \ref{claim:expectation-estimator-fourier}}

We first take the expectation of the WLSH estimator $\widetilde{k}(\x,\y)$ with respect to ${\z}\sim\text{Unif}([0,\w])$. By \eqref{estimator-fourier} we have,
\begin{equation*}
\begin{split}
&\EE_{{\z} \sim \text{Unif}([0,\w])} \left[\widetilde{k}(\x,\y)\right]\\ 
&= \mathbb{E}_{ {\z} }\left[\prod_{l=1}^{d} w_l \int_{[0,1/{\bf w}]} \wh{g}(\vxi)^* \wh{h}(\vxi) \, d\vxi\right]\\
&=\prod_{l=1}^{d} w_l\int_{[0,1/\w]} \sum_{{\bf j},{\bf j}'\in \ZZ^d} \mathbb{E}_{ {\z}}\left[ e^{+2\pi i {\z}^\top (\frac{\bf j - j'}{\w})} \right] \\
&\qquad \cdot e^{2\pi i \x^\top (\vxi-\frac{\bf j}{\w})} \wh{f}^{\otimes d}({\w \vxi-{\bf j}})\\ 
&\qquad \cdot e^{-2\pi i \y^\top (\vxi-\frac{{\bf j}'}{\w})}\cdot \wh{f}^{\otimes d}({\w \vxi-{\bf j}}) \, d\vxi.
\end{split}
\end{equation*}
Now if you take the expectation with respect to ${\z} \sim \text{Unif}([0,\w])$, by orthogonality, the only non-zero terms in the sum will correspond to the case when ${\bf j}={\bf j}'$. Hence,
\begin{align*}
&\EE_{{\z} } \left[\widetilde{k}(\x,\y)\right]\\ 
&= 
\prod_{l=1}^{d} w_l\int_{[0,1/\w]} \sum_{{\bf j}\in \ZZ^d} e^{+2\pi i \x^\top (\vxi-\frac{\bf j}{\w})} \wh{f}^{\otimes d}({\w \vxi-{\bf j}})\\ &\qquad\cdot e^{-2\pi i \y^\top (\vxi-\frac{{\bf j}}{\w})} \wh{f}^{\otimes d}({\w \vxi-{\bf j}}) d\vxi\\
&= \prod_{l=1}^{d} w_l\int_{\RR^d} e^{2\pi i (\x - \y)^\top \vxi} \left| \wh{f}^{\otimes d}({\w \vxi-{\bf j}}) \right|^2 d\vxi.
\end{align*}
Now taking the expectation of above with respect to $\w$ gives Claim \ref{claim:expectation-estimator-fourier}.
\end{proofof}

\begin{claim}[WLSH is Unbiased]\label{gen-kernel-unbiased}
	For any PDF $p(\cdot)$ with non-negative support and any even function $f: \RR \rightarrow \RR$ with support $[-1/2, 1/2]$ and $\|f\|_2=1$ if $\mathcal{H}$ is the LSH family as per Definition \ref{lsh-hashfunction-def} and $k(\cdot)$ is the WLSH kernel as in Definition \ref{def:gen-kernel} then for any $\x,\y \in \RR^d$ the following holds for the expectation of WLSH estimator (see Definition \ref{def:-lsh-estimator}),
	$$\EE_{h_{\w,\z} \sim \mathcal{H}} \left[\widetilde{k}(\x,\y)\right] = k(\x-\y).$$
\end{claim}
\begin{proof}
	The proof follows from Claim~\ref{claim:expectation-estimator-fourier} and Definition~\ref{def:gen-kernel}.
\end{proof}

\section{Omitted claims and proofs from Section~\ref{sec:smoothness}}

The following lemma follows from Theorem 5 of \cite{ghosal2006posterior}.
\begin{lem}
	For any shift-invariant kernel $k(\cdot)$, which has bounded mixed partial derivatives of up to fourth order, if $\eta: \RR \rightarrow \RR$ is a sample path from the Gaussian Process GP$(0, k(\x-\y))$, then for any $j \in [d]$, the derivative process $D_j\eta(x)$ is a Gaussian Process with zero mean, i.e., $\mathbb{E}[D_j\eta(\x)] = 0$ for every $\x \in \RR^d$, and the covariance $\mathbb{E}[ D_j\eta(\x) \cdot D_j\eta(\y) ] = -{D_j^2 k(\x-\y)}$ for every $\x,\y \in \RR^d$.
\end{lem}

The above lemma can be applied multiple times and extend to higher order derivative of GP.

\begin{cor}\label{cor:derivative-GP}
	For any positive integer $q$, any shift-invariant kernel $k(\cdot)$ which has bounded mixed partial derivatives of order up to $2q+2$, if $\eta: \RR \rightarrow \RR$ is a sample path from the Gaussian Process GP$(0, k(\x-\y))$, then for any $j \in [d]$ the $q^\text{th}$ order partial derivative process $D_j^q\eta(\x)$ is a Gaussian Process with zero mean and covariance $\mathbb{E}[ D_j^q\eta(\x) \cdot D_j^q\eta(\y) ] = (-1)^qD_j^{2q}k(\x-\y)$.
\end{cor}

The following lemma gives the derivatives of our WLSH kernel of Definition \ref{def:gen-kernel}. 

\begin{lem}\label{lem:derivative-gen-kernel}
	For any positive integer $q$, any $i \in [d]$, any function $f: \RR \rightarrow \RR$ with support $[-1/2, 1/2]$ and norm $\|f\|_2=1$ which is $\lceil q/2 \rceil$ times differentiable and any probability density function $p(\cdot)$ with non-negative support, if $k(\cdot)$ is the WLSH kernel as in Definition \ref{def:gen-kernel}, then for any $i \in [d]$ the following holds,
	\begin{align}
	&D_i^q k(\x) \label{derivative-kernel}\\
	&= \int_{\RR_{+}^d}  \frac{{p^{\otimes d}(\w)}}{w_i^q}  \left( D_i^{\lceil q/2 \rceil} f^{\otimes d}* D_i^{\lfloor q/2 \rfloor} f^{\otimes d}\right)(\frac{\x}{\w}) d\w,\nonumber
	\end{align}
	where $D_i^{j} f^{\otimes d}(\x) = \prod_{\substack{l\in[d]\\l\neq i}}f(x_l) \cdot f^{(j)}(x_i) $ for any integer $j \le \lceil q/2 \rceil$ and ${p^{\otimes d}(\w)} = \prod_{l=1}^d p(w_l)$.
\end{lem}

Therefore if the function $f(\cdot)$ is $\lceil q/2 \rceil$ times differentiable then $k(\cdot)$ will be $q$ times partially differentiable with respect to any coordinate. Hence, the LSH-able kernel $k(\cdot)$ inherits certain smoothness properties from the band-limited function $f(\cdot)$.

Now we use the result of Corollary \ref{cor:derivative-GP} to show that a GP$(0,k(\x-\y))$ with WLSH covariance kernel $k(\cdot)$ defined as in Definition \ref{def:gen-kernel} inherits its smoothness from the band-limited function $f(\cdot)$.

\begin{lem}\label{lem-covariance-deriv-process}
	For any positive integer $q$, any even function $f: \RR \rightarrow \RR$ with support $[-1/2, 1/2]$ which has bounded derivatives of order up to $q+1$, if $\eta: \RR^d \rightarrow \RR$ is a sample path from a GP$(0, k(\x-\y))$, where $k(\cdot)$ is the WLSH kernel as in Definition \ref{def:gen-kernel}, then for any $j \in [d]$, $D_j^q\eta(\x)$ is a Gaussian process with zero mean and covariance
	\begin{align*}
	&\mathbb{E}[ D_j^q\eta(x) \cdot D_j^q\eta(y) ]\\ 
	&= (-1)^q \int_{\RR_{+}^d} \frac{{p^{\otimes d}(\w)}} {w_j^{2q}}  \left( D_j^{ {q}} f^{\otimes d} * D_j^{{q}} f^{\otimes d} \right)\left(\frac{\x-\y}{\w}\right) \, d\w
	\end{align*}
	
	where $D_j^{q} f^{\otimes d}(\x) = \prod_{\substack{l\in[d]\\l\neq j}}f(x_l) \cdot f^{(q)}(x_j) $.
\end{lem}
\begin{proof}
	The proof follows from Corollary \ref{cor:derivative-GP} and Lemma \ref{lem:derivative-gen-kernel}.
\end{proof}

One can use Cauchy-Schwarz inequality to bound the covariance of $D^q\eta(\x)$ by the following,
$$\left| \mathbb{E}[ D_j^q\eta(\x) \cdot D_j^q\eta(\y) ] \right| \le \left\|D_j^q f^{\otimes d}\right\|_2^2 \cdot \int_{\RR_{+}} \frac{ p(w)}{w^{2q}} \, dw.$$
In particular, the derivative Gaussian Process $D_j^q\eta(\x)$ has the following variance, as long as the band-limited filter $f$ is a normalized function ($\|f\|_2=1$),
$$ \mathbb{E} \left[ |D_j^q\eta(\x)|^2 \right] = \left\|f^{(q)}\right\|_2^2 \cdot \int_0^\infty \frac{p(w)}{w^{2q}} \, dw.$$

Now we are ready to prove Lemma \ref{lem:smooth-GP-lsh-kernel}

\begin{proofof}{Lemma \ref{lem:smooth-GP-lsh-kernel}}
	It follows from multiple application of Lemma \ref{lem-covariance-deriv-process} that the derivative process ${\bf D} \eta(x)$ is a Gaussian process with zero mean and the covariance of ${\bf D} \eta(x)$ is the following,
	\begin{align*}
	&\mathbb{E}[ {\bf D}\eta(\x) \cdot {\bf D}\eta(\y) ]\\ 
	&= (-1)^q \int_{\RR_{+}^d} \prod_{l=1}^d {\frac{p(w_l)}{w_l^{2q_l}}}  \left( {\bf D} f^{\otimes d} * {\bf D} f^{\otimes d} \right)(\frac{\x-\y}{\w}) \, d\w\\
	&= (-1)^q \prod_{l=1}^d \int_{\RR_{+}} {\frac{p(w_l)}{w_l^{2q_l}}}  \left( D_l^{q_l} f * D_l^{q_l} f \right)(\frac{x_l-y_l}{w_l}) \, dw_l
	\end{align*}
	In order to show that the supremum of the Gaussian process ${\bf D}\eta(\x)$ has sub-Gaussian tail bound we use Proposition A.2.7 of \cite{van1996weak}. Let $\|\cdot\|_\rho$ denote the intrinsic semi-metric of the process ${\bf D}\eta(\x)$ which is defined as follows:
	\begin{align*}
	&\|\x-\y\|_\rho^2 = \mathbb{E}\left[ \left| {\bf D}\eta(\x) - {\bf D}\eta(\y) \right|^2 \right]\\
	&= 2  \prod_{l=1}^d \left(\left\|f^{(q_l)}\right\|_2^2 \int_{\RR_{+}} {\frac{p(w_l)}{w_l^{2q_l}}}  \, dw_l\right) \\
	& - 2(-1)^q \prod_{l=1}^d \int_{\RR_{+}} {\frac{p(w_l)}{w_l^{2q_l}}} \left( D_l^{q_l} f * D_l^{q_l} f \right)(\frac{x_l-y_l}{w_l}) \, dw_l.
	\end{align*}
	Since $f$ is an even function with bounded derivatives of order up to $q+1$, we have that 
	$$D_j \| \x \|_\rho^2|_{\x=0} = 0,$$ 
	for every $j \in [d]$ and also 
	$$D_j D_k \| \x \|_\rho^2|_{\x=0} = 0$$
	for every $j \neq k \in [d]$. Therefore, by Taylor's theorem we have the following,
	\begin{align*}
	&\|\x-\y\|_\rho^2 \le \sup_{\substack{ \z \in[0,1]^d \\ j \in [d]}} \left| D_j^2 \|\z\|_\rho^2 \right| \cdot \|\x-\y\|_2^2\\
	&= 2\sup_{{ j \in [d]}} \left| \prod_{\substack{ l \in [d] \\ l \neq j}} \left\|f^{(q'_l+\delta_{l,j})}\right\|_2^2 \int_{\RR_{+}} {\frac{p(w_l)}{w_l^{2(q_l+\delta_{l,j})}}} \right| \|\x-\y\|_2^2,
	\end{align*}
	where $\delta_{l,j} = 1$ iff $l = j$, and $\delta_{l,j} = 0$ otherwise.
	Therefore, the covering number of $[0,1]^d$ with respect to $\rho$ is bounded as follows:
	
	$$N\left(\epsilon, [0,1]^d, \rho\right) \le \left( \frac{C}{\epsilon} \right)^d,$$
	where 
	$$C = 2\sqrt{d} \cdot \sup_{{ j \in [d]}} \left| \prod_{\substack{ l \in [d] \\ l \neq j}} \left\|f^{(q'_l)}\right\|_2^2 \int_{\RR_{+}} {\frac{p(w_l)}{w_l^{2q'_l}}} \, dw_l \right|.$$ 
	Now using Proposition A.2.7 of \cite{van1996weak}, we have that 
	$$\Pr\left[ \sup_{x \in [0,1]^d} \left| {\bf D} \eta(x)\right| > M \right] \le \left(\frac{L M}{\sigma^2}\right)^d e^{-\frac{M^2}{\sigma^2}},$$
	where $\sigma^2 = \prod_{l=1}^d \left(\left\|f^{(q_l)}\right\|_2^2 \int_{\RR_{+}} {\frac{p(w_l)}{w_l^{2q_l}}}  \, dw_l\right)$ and $L = O\left( \sup_{{ j \in [d]}} \left| \prod_{\substack{ l \in [d] \\ l \neq j}} \left\|f^{(q_l+\delta_{l,j})}\right\|_2^2 \int_{\RR_{+}} {\frac{p(w_l)}{w_l^{2(q_l+\delta_{l,j})}}} \, dw_l \right| \right)$.
\end{proofof}

\section{Omitted lemmas and proofs from Section~\ref{sec:ose}} \label{sec:appD}
\begin{lem}[Running time and Memory of WLSH Kernel Matrix]
	For any positive integers $n,d$ and any dataset $\x^1,\x^2,\dots, \x^n \in \RR^d$, if $k_{f,p}(\cdot)$ is the WLSH estimator as in Definition~\ref{def:-lsh-estimator} and $\widetilde{K} \in \RR^{n \times n}$ is its corresponding kernel matrix then there exists an algorithm which using $O(dn)$ pre-processing time forms a data structure which can be stored using $O(n)$ memory words such that using this data structure, the product $\widetilde{K} \beta$ can be computed in time $O(n)$ for an arbitrary vector $\beta \in \RR^n$.
\end{lem}

\begin{proofof}{Claim \ref{claim:spectral-nomr-estimator}} First note that for any $\beta \in \RR^{n}$, we can write the quadratic form as,
	$$\beta^\top \widetilde{ K} \beta =\sum_{{\bf j} \in \ZZ^d} \left(\sum_{i=1}^n \beta_i \cdot f^{\otimes d} \left({\bf j} + \frac{{\z} - \x^i}{\w} \right) \right)^2 \ge 0$$	
	Also by Cauchy-Schwarz inequality we have, 
	\begin{align*}
	\beta^\top \widetilde{ K} \beta &=\sum_{{\bf j} \in \ZZ^d} \left(\sum_{i=1}^n \beta_i \cdot f^{\otimes d} \left({\bf j} + \frac{{\z} - \x^i}{\w} \right) \right)^2\\
	&\le \|f^{\otimes d}\|_\infty^2 \sum_{{\bf j} \in \ZZ^d} \left(\sum_{i: h_{\w,{\z}}(x^i) = {\bf j}} |\beta_i| \right)^2\\
	&\le \|f^{\otimes d}\|_\infty^2 \|\beta\|_1^2 \\
	&\le n \|f^{\otimes d}\|_\infty^2 \|\beta\|_2^2.
	\end{align*}
\end{proofof}

\begin{lem}\label{lem-matrix-chenoff}
	(Matrix Chernoff, \cite{DBLP:journals/focm/Tropp12}) Let $A_i\in \RR^{n\times n}$ be independent random positive semi-definite matrices satisfying $\mathbf{E}[\sum_i A_i]=I$ and for all $i$, $\|A_i\|_{op}\le \alpha$ with probability 1. Then for any $0<\epsilon<1$,
	$\Pr\left[\left\|\sum_i A_i - I\right\|_{op} \le \epsilon\right] \ge 1-2n \cdot  \exp\left( -\frac{\epsilon^2}{3\alpha} \right)$.
\end{lem} 

\begin{lem}[Slud's Inequality \cite{slud1977distribution}]\label{slud}
Let $X_1,X_2, \cdots X_n$ be iid Bernoulli random variables with $\Pr[X_i=1] = p$. If $p \le 1/4$ and $np \le r \le n$ or $np \le r \le (1-p)n$ then the following holds,
$$\Pr \left[ \sum_{i\in [n]} X_i \ge r\right] \ge \Pr\left[ Z \ge \frac{r - np}{\sqrt{np(1-p)}} \right],$$
where $Z$ is a normal random variable with zero mean and variance one.
\end{lem}
Therefore it follows from the above that, if $p \le 1/4$ and $t \ge 0$ or $p\le 1/2$ and $0\le t \le n(1-2p)$ then,

$$\Pr \left[ \sum_{i\in [n]} X_i - \mu \ge t\right] \ge \Pr\left[ Z \ge \frac{t}{\sqrt{np(1-p)}} \right],$$
where $\mu = \EE[\sum_{i\in [n]} X_i]$. This probability can be further lower bounded as follows:
$$\Pr \left[ \sum_{i\in [n]} X_i - \mu \ge t\right] \ge \frac{1}{4} e^{-2t^2/\mu}.$$

\begin{proofof}{Theorem \ref{thm:lowerbound}} Let the points $\{x^i\}_{i=1}^n \subseteq \RR^d$ be positioned as $x^1 =\cdots = x^{n/2} = (-\lambda/n, 0,0, \cdots 0)^\top$ and $x^{n/2+1} =\cdots = x^{n} = (\lambda/n, 0, 0, \cdots 0)^\top$.
	Let the vector $\beta \in \CC^n$ be defined as,
	$\beta_1 =\beta_2=\cdots = \beta_{n/2} = -1$
	and 
	$\beta_{n/2+1} =\dots = \beta_{n} = 1$. 
	The proof proceeds by showing that in order to preserve the quadratic form corresponding to this $\beta$, one needs to set $m =  \Omega \left( \frac{1}{\epsilon^2} \cdot \frac{n}{\lambda} \cdot \log n \right)$ and hence the lower bound follows for achieving an OSE.
	Let us compute the expectation of the quadratic form, 	
	\begin{align*}
	\EE_{h_{\w,\z} \sim \mathcal{H}}[\beta^\top \widetilde{ K}^s \beta]&= \beta^\top K \beta \\
	&=n^2 \left( 1 - \exp(-2\lambda/n) \right)/2.
	\end{align*}
	Now we compute the second moment of the quadratic form as follows:
	\begin{equation*}
	\begin{split}
	&\EE_{h_{\w,\z} \sim \mathcal{H}}\left[ (\beta^\top \widetilde{ K}^s \beta)^2 \right]\\
	&= \EE_{\w,{\z}} \left[ \left( \sum_{{\bf j} \in \ZZ^d} \left(\sum_{i=1}^n \beta_i \, \rect^{\otimes d} \left( \frac{{\bf j} \w - x^i + {\z}}{\w} \right) \right)^2 \right)^2\right]\\
		&= \EE_{w,{z}} \left[ \frac{n^4}{4} \cdot \mathbf{1}_{\left\{ |{z}| > \frac{w}{2}-\frac{\lambda}{n} \text{ or } w \le \frac{2\lambda}{n} \right\} }\right]\\
	&=\EE_{w}\left[ \frac{n^4}{4} \cdot \min\left( 1,\frac{2\lambda}{nw} \right) \right]= \frac{n^4}{4} \cdot \left( 1 - e^{-2\lambda/n} \right).
	\end{split}
	\end{equation*}
	
	Hence, we have the following for the ratio of second moment to the square of the first moment,
	$$\frac{\EE\left[ (\beta^\top \widetilde{ K}^s \beta)^2 \right]}{\EE\left[ \beta^\top \widetilde{ K}^s \beta \right]^2} = \frac{1}{1 - e^{-2\lambda/n}} \ge \frac{n}{2\lambda}.$$
	
	Note that the LSH estimator $\beta^\top \widetilde{ K}^s \beta$ for this particular dataset $x_1, \cdots, x_n$ and vector $\beta$ take in two possible values, zero and $\frac{n^2}{2}$. Therefore,
	$$\beta^\top \widetilde{ K}^s \beta = \begin{cases}
	\frac{n^2}{2} & \text{with probability } p \le  \frac{2 \lambda}{n}\\
	0 & \text{with probability } 1-p
	\end{cases}.$$
	Note that $(1+\epsilon)\beta^\top K \beta + \epsilon\lambda \|\beta\|_2^2 \le (1+3\epsilon)\beta^\top K \beta $. Since $p \le 1/4$, by using Slud's inequality, Lemma \ref{slud}, the probability of guaranteeing that $\frac{1}{m} \sum_{s=1}^m  \beta^\top \widetilde{ K}^s \beta \le (1+\epsilon) \beta^\top K \beta + \epsilon\lambda \|\beta\|_2^2 $ is bounded as follows:
	\begin{align*}
	&\Pr\left[ \frac{1}{m} \sum_{s=1}^m  \beta^\top \widetilde{ K}^s \beta > (1+\epsilon) \beta^\top K \beta + \epsilon\lambda \|\beta\|_2^2 \right]\\ 
	&\ge \Pr\left[ \frac{1}{m} \sum_{s=1}^m  \beta^\top \widetilde{ K}^s \beta > (1+3\epsilon) \beta^\top K \beta \right]\\
	&= \Pr\left[ \frac{1}{m} \sum_{s=1}^m \beta^\top \widetilde{ K}^s \beta > (1+3\epsilon) \cdot \EE\left[ \beta^\top \widetilde{ K} \beta \right] \right]\\
	&\ge \frac{1}{4} e^{-\frac{2 (3\epsilon p m )^2}{p m}}= \frac{1}{4} e^{-18\epsilon^2 p m} \ge \frac{1}{4} e^{-36\epsilon^2 \frac{\lambda}{n} m}.
	\end{align*}
	Therefore, in order to have $\Pr[ \frac{1}{m} \sum_{s=1}^m  \beta^\top \widetilde{ K}^s \beta > (1+\epsilon) \beta^\top K \beta + \epsilon\lambda \|\beta\|_2^2 ] < \frac{1}{n}$, we need to average at least $m = \Omega(\frac{n}{\lambda} \cdot \log n / \epsilon^2)$ independent instances of WLSH estimator.
\end{proofof}

\section{Risk bound of approximate KRR via LSH-Estimator} \label{sec:appF}
We use risk bounds to analyze the quality our approximate KRR estimator. It is common to bound the expected in-sample predication error of the KRR estimator as an empirical estimate of the statistical risk \cite{avron2017random, bach2013sharp, alaoui2015fast, musco2017recursive}. Formally, the \emph{empirical risk} of an estimator ${ \eta}$ is defined as,
$$\R({ \eta}) = \EE_{\{\epsilon_i\}} \left[ \frac{1}{n} \sum_{i=1}^n \left( {\eta}(\x^i) - \eta^*(\x^i) \right)^2 \right].$$
Suppose $\eta(\cdot)$ is the exact KRR estimator using kernel function $k(\cdot)$. Also suppose that $\widetilde{\eta}(\cdot)$ is the regressor obtained by solving the approximate KRR problem using the approximate kernel function $\widetilde{ k}(\cdot)$.
The following Lemma bounds the excess risk of approximate KRR estimator $\widetilde{\eta}(\cdot)$.
\begin{lem}[Approximate KRR Empirical Risk Bound]\label{lem:risk-bound}
	Let $\eta(\cdot)$ be the exact KRR estimator using the WLSH kernel function $k(\cdot)$ (Definition \ref{def:gen-kernel}). Suppose $\widetilde{ k}^s(\cdot)$ are independent instances of WLSH estimator for all $s \in [m]$. Let $\widetilde{\eta}$ be the approximate KRR estimator obtained by using the approximate kernel function $\widetilde{ k}(\cdot) := \frac{1}{m} \sum_{s=1}^m \widetilde{ k}^s(\cdot)$ and let $\widetilde{ K}$ be the corresponding kernel matrix to $\widetilde{ k}(\cdot)$. If $m = \Omega\left( \frac{\|f^{\otimes d}\|_\infty^2}{\epsilon^2} \cdot ({n}/{\lambda})\cdot \log n\right)$ then the following holds\footnote{When we hash $n$ points using LSH, we expect the number of non-empty buckets to grow at a lower rate than $n$. Therefore, we expect to have $\frac{\mathrm{rank}(\widetilde{ K})}{n} \rightarrow 0$ as $n$ grows.},
	$$\Pr\left[\R(\widetilde{ \eta}) \le \frac{\R({ \eta})}{1-\epsilon} + \frac{\epsilon \sigma_\epsilon^2 \cdot \mathrm{rank}(\widetilde{K})}{(1+\epsilon)n} \right] \ge 1 - \frac{1}{\mathrm{poly}(n)}.$$
\end{lem}
\begin{proof}
	First note that Theorem \ref{lem:OSE} implied that with probability $1- \frac{1}{\mathrm{poly}(n)}$ the approximate kernel matrix satisfies following spectral guarantee,
	$$(1-\epsilon)(K + \lambda I_n) \preceq \widetilde{ K} + \lambda I_n \preceq (1+\epsilon) (K + \lambda I_n).$$
	Therefore the lemma follows directly from invoking Lemma 2 of \cite{avron2017random}.
\end{proof}

\end{document}